\documentclass[10pt,journal,compsoc]{IEEEtran}

\newtheorem{definition}{Definition}
\newtheorem{strategy}{Strategy}    
\newtheorem{theorem}{Theorem}
\newenvironment{proof}{\begin{IEEEproof}}{\end{IEEEproof}}
\usepackage{multirow}
\usepackage{amsmath}
\usepackage{array}
\newtheorem{upper bound}{Upper bound}
\usepackage{url}

\ifCLASSINFOpdf
   \usepackage[pdftex]{graphicx}
\else
\fi

\usepackage[linesnumbered,ruled]{algorithm2e} 

\usepackage[caption=false,font=footnotesize,labelfont=sf,textfont=sf]{subfig}
  \usepackage[caption=false,font=footnotesize]{subfig}

\newcommand\MYhyperrefoptions{bookmarks=true,bookmarksnumbered=true,
	pdfpagemode={UseOutlines},plainpages=false,pdfpagelabels=true,
	colorlinks=true,linkcolor={blue},citecolor={blue},urlcolor={blue},
	pdftitle={Discovering Utility-driven Interval Rules},
	pdfsubject={Typesetting},
	pdfauthor={ChunKai Zhang},
	pdfkeywords={artificial intelligence, knowledge discovery, utility mining, interval rule, interval event}}
\ifCLASSINFOpdf
\usepackage[\MYhyperrefoptions,pdftex]{hyperref}
\else
\usepackage[\MYhyperrefoptions,breaklinks=true,dvips]{hyperref}
\usepackage{breakurl}
\fi

\begin{document}
%

\title{Discovering Utility-driven Interval Rules}

\author{Chunkai Zhang, Maohua Lyu, Huaijin Hao, Wensheng Gan*, and Philip S. Yu,~\IEEEmembership{Life Fellow,~IEEE}
	
	\IEEEcompsocitemizethanks{\IEEEcompsocthanksitem Chunkai Zhang, Maohua Lyu, and Huaijin Hao are with the School of Computer Science and Technology, Harbin Institute of Technology (Shenzhen), Shenzhen 518055, China. (E-mail: ckzhang@hit.edu.cn, 21s151083@stu.hit.edu.cn, 22s151144@stu.hit.edu.cn)

	\IEEEcompsocthanksitem Wensheng Gan is with the College of Cyber Security, Jinan University, Guangzhou 510632, China. (E-mail: wsgan001@gmail.com) 
	
	\IEEEcompsocthanksitem Philip S. Yu is with the Department of Computer Science, University of Illinois at Chicago, IL, USA. (E-mail: psyu@uic.edu)}

	\thanks{Corresponding author: Wensheng Gan}
}

\IEEEtitleabstractindextext{%

\begin{abstract}
    For artificial intelligence, high-utility sequential rule mining (HUSRM) is a knowledge discovery method that can reveal the associations between events in the sequences. Recently, abundant methods have been proposed to discover high-utility sequence rules. However, the existing methods are all related to point-based sequences. Interval events that persist for some time are common. Traditional interval-event sequence knowledge discovery tasks mainly focus on pattern discovery, but patterns cannot reveal the correlation between interval events well. Moreover, the existing HUSRM algorithms cannot be directly applied to interval-event sequences since the relation in interval-event sequences is much more intricate than those in point-based sequences. In this work, we propose a utility-driven interval rule mining (UIRMiner) algorithm that can extract all utility-driven interval rules (UIRs) from the interval-event sequence database to solve the problem. In UIRMiner, we first introduce a numeric encoding relation representation, which can save much time on relation computation and storage on relation representation. Furthermore, to shrink the search space, we also propose a \textit{utility complement pruning} strategy, which incorporates the utility upper bound with the relation. Finally, plentiful experiments implemented on both real-world and synthetic datasets verify that UIRMiner is an effective and efficient algorithm. 
\end{abstract}

\begin{IEEEkeywords}
 artificial intelligence, knowledge discovery, utility mining, interval rule, interval event.
\end{IEEEkeywords}}

\maketitle

\IEEEdisplaynontitleabstractindextext

\IEEEpeerreviewmaketitle

\IEEEraisesectionheading{
\section{Introduction}}
\label{sec:introduction}

\IEEEPARstart{S}{equential} pattern mining (SPM) \cite{agrawal1995mining, fournier2017surveys, gan2019survey, 2004Mining} is a research field for the discovery of underlying, valuable, and significant knowledge from a sequence of events database according to the support of patterns, where support refers to the frequency of a pattern occurring in the database. Correspondingly, the SPM has been used in a variety of real-world applications, including market analysis \cite{srikant1996mining}, bio-informatics \cite{wang2007frequent}, and web page click analysis \cite{fournier2012using}. To further apply these discovered sequential patterns, sequential rule mining (SRM) has been developed well in the last decades, which takes into account the probability that a pattern will be followed \cite{fournier2015mining, fournier2012cmrules, zaki2001spade}. A sequential rule (SR) $r$ consists of $X$ $\rightarrow$ $Y$, where $X$ is the antecedent of $r$ and $Y$ is the consequent of $r$.  To find SRs, two measures are generally used: support and confidence. The support of a SR $r$ $=$ $X$ $\rightarrow$ $Y$ is how many sequences in a sequential database that $r$ appears. The confidence of a SR $r$ is the support of the rule divided by the number of sequences containing the sequence $X$. Sequential rules can provide more semantic information compared to sequential patterns since they can be understood as the conditional probability $P$($Y$$|$$X$), as SR provides an explicit prediction probability compared to sequential patterns when the patterns are used for sequence prediction. SRM is widely applied to various applications, such as alarm sequence analysis \cite{ccelebi2014alarm} and restaurant recommendation \cite{han2013mining}.

The above algorithms for SPM and SRM are applied to sequence databases composed of point-based events occurring in an instant, such as the sequences of transaction records. However, event occurrences are generally associated with time intervals, which can be named interval-based events, temporal events, or interval events. For example, in a rental field, goods or services are rented out for a period of time (or interval duration); in the medical field, the symptoms of the diseases and medical treatments tend to persist for periods; in the smart home, the appliances run for periods. An interval event pattern consists of several interval events and their temporal relation in terms of Allen’s temporal logic \cite{allen1983maintaining} due to varying lengths of time intervals of events. There are seven possible temporal relations between a pair of interval events \cite{allen1983maintaining}, which can introduce richer information to help us better understand the underlying information in rich data. Therefore, it is necessary to explore the knowledge discovery of interval-based events.

Temporal pattern mining (TPM) is a way to discover all frequent interval-based sequences, named interval patterns or temporal patterns \cite{kam2000discovering, khoshnevisan2018recent, lee2020mining, novitski2020all, sharma2018stipa, xie2018mining}. For example, in a rental field, we discover two interval patterns: a 2-temporal pattern \{A, B\} with relation A \textbf{contain} B and a 3-temporal pattern \{A, B, D\} with relation A \textbf{contain} B, A \textbf{overlaps} D, and B \textbf{before} D, where renting good A is a car, B is GPS, and D is a room. For the 2-temporal pattern, we know that when renting a car for a period of time, some customers will rent GPS. For the 3-temporal pattern, we can discover that during renting a car for a period of time, a customer will rent GPS, and after the rental period GPS, he will rent a room. However, we do not know the proportion of these customers renting a room after the rental period of GPS among those customers that rent GPS while renting a car. That is, we can only know approximately the predicted interval events with temporal relations, but we cannot give the probability of the interval events occurring.

Interval rule mining (IRM) and temporal rule mining were proposed to tackle such tasks. Hoppner \cite{hoppner2001learning} presents the issue of discovering interval rules from a state sequence, utilizes a square matrix $R$ whose element $R$[$E_i$, $E_j$] denotes the relation between state intervals $E_i$ and $E_j$, and designs an Apriori-like algorithm \cite{agrawal1994fast} to find out rules from an interval-based event database. ARMADA \cite{winarko2007armada} adopts a memory-based indexing algorithm based on MEMISP \cite{lin2002fast} to extract temporal patterns and interval rules. However, these algorithms focus on the interval rule and do not provide an efficient mining algorithm. They all employ a two-stage method, which first finds out all frequent temporal patterns and then generates interval rules. It may not be efficient since the algorithms will waste a lot of time and memory generating many temporal patterns that are irrelevant to the interval rules. Therefore, mining interval rules from interval events remains a challenging problem. Besides, they are support-based mining algorithms, which only consider the temporal relations of interval events in an interval rule and check the interval rule's frequency. They neglect the length of the duration of each interval event in an interval rule and user preferences on interval events. As illustrated in Fig. \ref{Introexample}, for instance, in rental services, there are two interval event sequences, $S_1$ and $S_2$. $S_1$ contains two interval events A and B with relation A \textbf{contain} B. In $S_2$, there are three interval events A, B, and D with relations A \textbf{contain} B, A \textbf{before} D, and B \textbf{before} D. A lasts for 0.5 hours and B for 0.3 hours in $S_1$, while A, B, and D last respectively for 5, 4, and 10 hours in $S_2$. Obviously, the interval events occurring in $S_2$ can bring high profit to a rental company, but IRM treats \{A, B\} with relation A \textbf{contain} B as having the same importance. IRM will output two rules: $r_1$ = \{A\} $\rightarrow$ \{B\} with relation A \textbf{contain} B and $r_2$ = \{A, B\} $\rightarrow$ \{D\} with relations A \textbf{contain} B, A \textbf{before} D, and B \textbf{before} D if the minimum support is 2 and the minimum confidence is 0.5. However, $r_1$ in some E-sequences cannot bring high profit, such as $S_1$ in Fig. \ref{Introexample}. Thus, IRM cannot reflect the rules that users are really interested in.

\begin{figure}[h]
    \centering
    \includegraphics[width=0.94\linewidth]{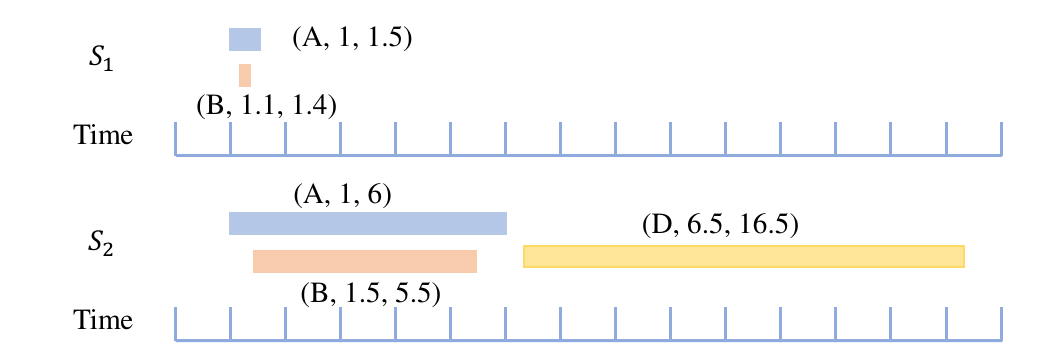}
    \caption{The different importance of the same relation in interval events.}
    \label{Introexample}
\end{figure}

Utility-driven interval rule mining can consider the duration of interval events and the user's preferences to derive interval rules that are of great interest to the user. Like in the example above, we can find a high-utility interval rule $r_2$ with a confidence of 0.5 and filter low-utility rule $r_1$. Based on interval events A \textbf{contain} B, we can infer that an interval D will occur with a probability of 0.5 in the future. Thus, the user can recommend interval event D after A and B for a high profit. Although there are several works \cite{huang2019mining, wang2020mining, mirbagheri2021mining} related to utility-driven interval pattern mining, they cannot discover interval rules directly unless using a two-stage mining approach to extract interval rules, resulting in lots of irrelevant patterns generated and wasting much of time. Besides, the existing IRM algorithms \cite{hoppner2001learning, winarko2007armada} cannot perform utility-driven interval rule mining. They do not consider utility for mining, such that the pruning strategies cannot work in utility-driven scenarios, which makes the existing IRM algorithms inapplicable to the utility domain. In addition, existing IRM algorithms use a matrix for relation representation \cite{hoppner2001finding}, which requires $O$($n^2$) time complexity to compute the relations and $O$($n^2$) space complexity for relation storage between interval events. The existing utility-driven SRM algorithms \cite{zida2015efficient, huang2021us, zhang2022totally} only consider sequential relations based on point-based events and do not take into account the temporal relations between interval events. Besides, the pruning strategies used in utility-driven SRM couldn't efficiently handle a sea of candidate interval rules when an interval rule was extended.

For the need to exploit utility-driven interval rule mining, we propose an algorithm called UIRMiner. In order to quickly calculate the relations between interval events and compress the file size of the output results, we first propose a numerical encoding relation representation that reduces the time required for relation computation while saving significant storage space. In utility-driven mining, the design of the pruning strategy is usually related to utility upper bounds according to the corresponding remaining utility. In UIRMiner, we utilize the utility upper bounds \textit{LERSPEU} and \textit{RERSPEU} introduced in \cite{zhang2022totally}. In addition, according to the feature of relation, we design a novel pruning strategy called utility complement pruning. As the rules have two evaluation conditions: utility and confidence, we also use confidence to prune. Different from the data structure, utility table, widely used in utility-driven SRM \cite{zida2015efficient, huang2021us, zhang2022totally}, we also redesign the data structures and the corresponding projected database. To summarize, the major contributions of this paper are listed below:

\begin{itemize}
    \item	We formulate the problem of discovering interval rules and propose a novel algorithm, namely UIRMiner, that can efficiently find all interval rules in a given database.
 
    \item 	A numerical encoding relation representation is introduced. It can compress the relations and reduce the storage size of the output file. When extending an interval, it can determine the relations quickly.

    \item   For shrinking search space, we also introduce a novel pruning strategy namely the utility complement pruning strategy. With this pruning strategy and the existing pruning strategies, we can efficiently discover all interval rules.
 
    \item 	A lot of experiments are conducted on both real and synthetic datasets. The experimental results show that our method is effective and much more efficient than those without any optimization methods.
\end{itemize}

The remaining parts of this paper are as follows: We review the related work about TPM, utility-driven SPM, and utility-driven SRM in Section \ref{sec:relatedwork}. In Section \ref{sec:preliminaries}, we introduce the necessary definitions and the problem statement. The specific method and the strategies are given in Section \ref{sec:method}. We analyze the experimental results in Section \ref{sec:experiments}. In Section \ref{sec:conclusion}, the conclusions of this paper are presented.

\section{Related work}  \label{sec:relatedwork}

In this section, we review the former work on interval mining, utility-driven pattern mining, and utility-driven rule mining, which are most relevant to this paper.

\subsection{Interval mining}

Interval pattern mining was proposed by Kam \textit{et al.} \cite{kam2000discovering}. In the early stages of the development of interval pattern mining, the research mainly focused on the representation \cite{wu2007mining, hoppner2001finding, chen2015mining} of intervals based on Allen's relations \cite{allen1983maintaining}. There are two types of relation representation: matrix representation and endpoint representation. Usually, matrix representation is the most intuitive way to express relations. Relations between any two intervals in an interval event sequence are stored in a matrix. Thus, there are many algorithms \cite{moskovitch2015fast, lee2020mining, sharma2018stipa, lee2020z} that improve the efficiency of mining interval patterns using a matrix representation. They take advantage of the transitivity property of the relation between interval events to reduce candidate generation. However, it will cost \textit{O}($n^2$) time complexity to compute the relations and take \textit{O}($n^2$) space complexity to store them, where $n$ is the number of intervals in an interval pattern. For mining like point-based sequences, endpoint representation has also been well-developed \cite{wu2007mining, chen2015mining}. Endpoint representation transforms each interval into an endpoint pair according to the timestamp of the corresponding interval's endpoint and forms point-based sequences. Like point-based sequences, endpoint representation does not have to keep track of the relationships between intervals. This means that the pruning strategies used for sequential pattern mining can be used for endpoint representation interval events. However, the search space for endpoint representation will increase dramatically since each interval owns two endpoints. Coincidence event-set representation (CER) \cite{mirbagheri2020high} is also a variant of endpoint representation. CER will divide interval events into sub-interval events according to the endpoint of interval events first, then transform the sub-interval events into point events according to the starting time for each sub-interval event. In other words, CER turns the original interval event into a point event. Thus, the relationship is no longer to be maintained.

The above methods all concentrate on pattern mining. Interval rule mining \cite{hoppner2001finding, hoppner2001discovery, hoppner2001learning, winarko2007armada} can discover patterns with confidence to predict the future sequence and reveal the correlation between intervals. These methods mainly focus on the definition of the interval rule, and the existing algorithms have all adopted a two-phase mining method. They first discover all frequent interval patterns and then generate interval rules according to the frequent interval patterns, causing a lot of interval patterns that are irrelevant to the interval rules being discovered and wasting a lot of time. For example, ARMADA \cite{winarko2007armada} uses an index data structure to maintain the position of each interval event in the sequence, which avoids multiple database scanning and then generates interval patterns by merging the index structures. However, all of these methods are frequency-driven and use the Apriori property \cite{agrawal1994fast} to shrink search space, making the pruning strategies inapplicable to utility-driven mining.

\subsection{Utility-driven pattern mining}

Usually, utility consists of two aspects: internal utility (\textit{iu)} and external utility (\textit{eu}). In point-based event mining, the external utility is constant for each event \cite{gan2020proum, gan2020fast, wang2016efficiently, yin2012uspan}, and the actual utility of an event is \textit{iu} $\times$ \textit{eu}. However, the \textit{eu} should vary over time in some scenarios. In interval mining, \textit{iu} can be viewed as the duration of the event. The actual utility of the event should therefore be the integral of the external utility over the internal utility. For example, in the smart home, the duration of use of an appliance and the corresponding power can derive the utility of the energy consumed by the appliance; in a computation server, the execution of a thread lasts for some time and the consumption of system resources accounts for each moment, where internal utility is the duration and external utility is the resources accounts for each moment.

Based on the utility concept, utility-driven pattern mining (UPM) \cite{yin2012uspan, gan2021survey,tran2020statistically} can discover patterns according to user preference. USpan \cite{yin2012uspan} is the first algorithm that can discover high-utility sequences in complex sequences. Although some studies \cite{shie2011mining, ahmed2010mining} proposed utility sequence mining before USpan, they cannot handle the complex sequence. After that, there are many efficient and effective UPM algorithms \cite{wang2016efficiently, gan2020fast} based on USpan. In point-based sequences, algorithms usually use different utility upper bounds to prune the search space, such as \textit{SWU}, \textit{PEU}, and \textit{RSU}. \textit{SWU} is the most basic utility upper bound, which counts the total utility of the sequences in which a candidate sequential pattern appears. Thus, \textit{SWU} has limited pruning capability. To prune the search space well, HUS-Span \cite{wang2016efficiently} proposed \textit{PEU} and \textit{RSU}. \textit{PEU} and \textit{RSU} calculate the utility upper bound according to the remaining utility. Therefore, a tighter utility upper bound can be produced. In addition, HUSP-SP \cite{zhang2022tusq} introduced a new utility upper bound called \textit{TRSU}, which is the tightest utility upper bound. Furthermore, there are many application-oriented methods \cite{zhang2021shelf, zhang2022tusq, gan2021explainable} about utility-driven point-based sequence mining. However, for exploring interval-based event sequences, HU-TIED \cite{huang2019mining}, HUTPMiner \cite{wang2020mining}, and HUIPMiner \cite{mirbagheri2021mining}, have been studied in the literature. HU-TIED first extracts all frequent interval patterns and then finds high-utility interval patterns according to the frequent interval patterns. HU-TIED cannot discover all high-utility interval patterns since there are some patterns that are high-utility but low-frequency. HUTPMiner is an algorithm that uses endpoint representation to mine high-utility interval patterns. HHUTPMiner captures high-utility interval patterns like utility-driven point-based sequence pattern mining, but it needs endpoints to happen in pairs to make sure events are complete. HUIPMiner \cite{mirbagheri2020high} uses CER to extract high-utility interval patterns and can apply many pruning strategies used in USpan and HUSP-SP.

\subsection{Utility-driven rule mining}

In addition to utility-driven pattern mining, utility-driven sequential rule mining is also a key research area. High-utility sequential rule mining (HUSRM) \cite{zida2015efficient} was the first utility-driven sequential rule mining algorithm, which used two utility upper bounds, \textit{LEPEU} and \textit{REPEU}, to shrink the search space. \textit{LEPEU} and \textit{REPEU} are calculated separately based on the remaining utility of events that can be extended into the antecedent and consequent. To maintain the information of potential sequential rules, a data structure named utility table was also introduced by \cite{zida2015efficient}. Based on HUSRM, US-Rule introduced two tighter utility upper bounds, \textit{LERSU} and \textit{RERSU}, to efficiently find all sequential rules. Both HUSRM and US-Rule use the partially-ordered sequential rule, which will cause some problems in sequences where the timing relation cannot be changed. TotalSR \cite{zhang2022totally} introduces two novel utility upper bounds called \textit{LERSPEU} and \textit{RERSPEU}, which are tighter than \textit{LERSU} and \textit{RERSU} respectively, to discover totally ordered sequential rules solving the problem in HUSRM and US-Rule. Besides, TotalSR proposed a pruning strategy regarding confidence since rule mining has two measures: utility and confidence. In particular, utility-driven sequential rule mining can be applied to other fields such as anomaly detection \cite{gan2021anomaly} and gene rule discovery \cite{segura2022mining}. However, no literature focuses on utility-driven interval rules.
\section{Background}   \label{sec:preliminaries}

In this section, we give some important definitions related to this paper. Then, we formulate the problem of interval rule mining.

\subsection{Notations and definitions}

\begin{definition}[Interval with utility]
  \rm  Let $\Sigma$ $=$ \{$e_1$, $e_2$, $\cdots$, $e_m$\} be a set of $m$ event symbols. An interval event is composed of a four-tuple: \textit{E} $=$ ($e$, $st$, $ft$, $u$), where $e$ $\in$ $\Sigma$ is the symbol of this interval event, $st$ is the start time of this event, \textit{ft} is the finish time of this event, and \textit{u} is its utility. For brevity, we refer to intervals as "interval events". We use $E.e$, $E.st$, $E.ft$, and $E.u$ to represent the interval's type, start time, finish time, and utility, respectively.
\end{definition}

In this paper, we use the actual utility of the interval directly rather than deriving it indirectly from the internal and external utilities.

\begin{definition}[Order of intervals and E-sequence]
    \rm   Given two interval events $E_1$ and $E_2$, we say $E_2$ will follow $E_1$ if $E_1.st$ $\textless$ $E_2.st$, otherwise if $E1.st$ $=$ $E_2.st$ and $E_1.ft$ $\textless$ $E_2.ft$, and otherwise if $E_1.st$ = $E_2.st$ and $E_1.ft$ $=$ $E_2.ft$ and $E_1.e$ $\prec_{lex}$ $E_2.e$. Note that $E_1.e$ $\prec_{lex}$ $E_2.e$ means the event will follow lexicographic order \cite{lee2020z}. An E-sequence $S$ is a set of ordered interval events. 
\end{definition}

\begin{definition}[Size of E-sequence]
   \rm Let \textit{S} be an E-sequence. The size of \textit{S} is defined as the number of interval events that occur in \textit{S}.
\end{definition}

\begin{definition}[Interval database]
   \rm   An interval database $\mathcal{D}$ is a set of E-sequences $<$$S_1$, $S_2$, $\cdots$, $S_n$$>$, where $S_i$ $(1$ $\le$ $i$ $\le$ $n)$ is an E-sequence with a particular identifier $i$.
\end{definition}

We use Allen's relations \cite{allen1983maintaining} to represent the relation between two intervals in this paper. Allen's relations stipulate 13 relations that two intervals can form. We only use seven relations to present the relation between two intervals since the other 6 relations are redundant, as shown in Fig. \ref{relation}. Besides, we use a numerical representation to describe each relation. For brevity, we use the first letter of each relation to represent it. For example, \textbf{b} signifies \textbf{before}, \textbf{o} represents \textbf{overlaps}.

\begin{definition}[Temporal relation between two intervals]
    Given two intervals $E_{1}$ and $E_{2}$, the relation between $E_{1}$ and $E_{2}$ is defined as one of the Allen's relations \cite{allen1983maintaining} and denoted as $R$($E_{1}$, $E_{2}$).
\end{definition}

\begin{figure}[h]
    \centering
    \includegraphics[width=1.02\linewidth]{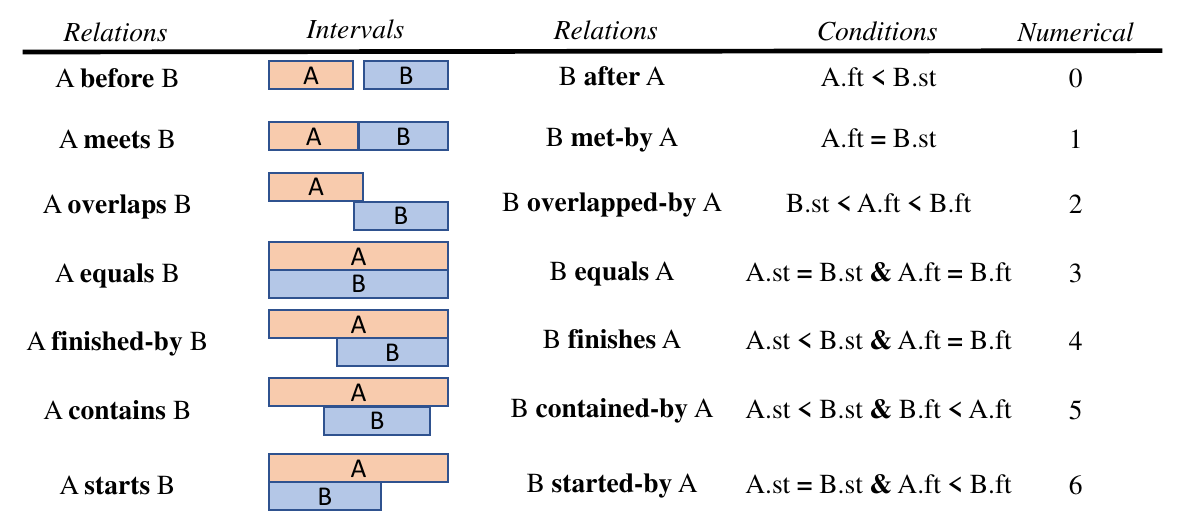}
    \caption{Allen's relations \cite{allen1983maintaining} with numerical representation.}
    \label{relation}
\end{figure}

\begin{definition}[Interval rule]
    \label{definition:IR}
    \rm An interval rule (IR) is defined as $r$ $=$ $X$ $\rightarrow$ $Y$, where $X$ and $Y$ are E-sequences with both sizes greater than or equal to 1, and $X$ $\cap$ $Y$ $=$ $\emptyset$. Meanwhile, the start time of each interval in $X$ should be earlier or less than the start time of intervals in $Y$. In addition, $X$ and $Y$ are the antecedent and consequent of $r$, respectively. Note that the start time of intervals in $X$ is less than intervals in $Y$, which ensures that $X$ can infer $Y$.
\end{definition}

\begin{definition}[Temporal relations between intervals in an interval rule]
    \label{definition:IR}
    \rm Given an IR $r$ $=$ $X$ $\rightarrow$ $Y$, its antecedent size is $k$ and the size of $Y$ is $n$. Thus, the relations of $r$ are described as $\mathcal{R}$ $=$ \{$R$($E_{1}$, $E_{2}$), $\cdots$, $R$($E_{k-1}$, $E_{k}$), $\dots$, $R$($E_{k+n-1}$, $E_{k+n}$)\}. Traditionally, $\mathcal{R}$ is represented using a matrix.
\end{definition}

\begin{definition}[IR occurrence]
    \rm Given an IR $r$ $=$ $\{$$E_1$, $\cdots$, $E_{k}$$\}$ $\rightarrow$ $\{$$E_{k+1}$, $\cdots$, $E_{n}$$\}$ with relations $\mathcal{R}$ $=$ \{$R$($E_{1}$, $E_{2}$), $\cdots$, $R$($E_{n-1}$, $E_{n}$)\} and an E-sequence $S$ $=$ $\{$${E^\prime}_1$, ${E^\prime}_2$, $\cdots$, ${E^\prime}_m$$\}$, we say that $r$ appears in $S$ if and only if there exists integers $1$ $\le$ $j_1$ $\textless$ $j_2$ $\textless$ $\cdots$ $\textless$ $j_n$ $\le$ $m$ such that $E_1$.$e$ $=$ ${E^\prime}_{j_1}$.$e$, $E_2$.$e$ $=$ ${E^\prime}_{j_2}$.$e$, $\cdots$, $E_n$.$e$ $=$ ${E^\prime}_{j_n}$.$e$ and $\forall$ $1$ $\le$ $a$ $\textless$ $b$ $\le$ $n$, $R$($E_{a}$, $E_{b}$) = $R$(${E^\prime}_{j_a}$, ${E^\prime}_{j_b}$). Moreover, we use \textit{seq}(\textit{r}) to represent the set of E-sequences that IR $r$ occurs and \textit{ant}(\textit{r}) to represent the set of E-sequences that the antecedent of $r$ appears.
\end{definition}

\begin{definition}[Support and confidence]
    \rm Given an IR and an interval database $\mathcal{D}$, the support of $r$ is defined as $\lvert$$seq(r)$$\lvert$ / $\lvert$$\mathcal{D}$$\rvert$, which equals the number of E-sequences containing $r$ divides by the number of E-sequences in $\mathcal{D}$. We use \textit{sup}($r$) to denote it. The confidence of IR $r$ is defined as \textit{conf}($r$) $=$ $\lvert$$seq(r)$$\lvert$ / $\lvert$$ant(r)$$\rvert$, which signifies that the confidence of IR $r$ is equal to the number of E-sequences containing $r$ divided by the number of E-sequences that $r$'s antecedent occurs \cite{zida2015efficient, huang2021us, zhang2022totally}.
\end{definition}

\begin{definition}[Utility of an IR in an E-sequence and in a database]	
    \rm Given an IR $r$ and an E-sequence $S_k$, the utility of $r$ in $S_k$ is denoted as $u$($r$, $S_k)$, which is the sum of each interval's utility in $r$ and $r$ occur in $S_k$, i.e., $u$($r$, $S_k)$ $=$ $\sum_{E \in r \land S_k \subseteq \textit{seq}(r)}$ $E.u$. Given an IR $r$ and an interval database $\mathcal{D}$, the utility of $r$ in $\mathcal{D}$ is defined as $u$($r$) $=$ $\sum_{S_k \in \textit{seq}(r) \land \textit{seq}(r) \subseteq \mathcal{D}}$ $u$($r$, $S_k)$ \cite{zida2015efficient, huang2021us, zhang2022totally}.
\end{definition}

\subsection{Problem statement}

\begin{definition}[Utility-driven interval rule]	
    \rm Given a minimum utility threshold \textit{minutil} and a minimum confidence threshold \textit{minconf} $\in$ [0, 1], the utility-driven interval rule (UIR) is the IR that satisfies \textit{minutil} and \textit{minconf}, simultaneously \cite{zhang2022totally, huang2021us, zida2015efficient}.
\end{definition}

\textbf{Problem statement:} Given an interval database $\mathcal{D}$, according to the definition of utility-driven interval rule discovery, the problem of utility-driven interval rule discovery is to find all UIRs in $\mathcal{D}$.

\begin{definition}[Interval rule extension]	
    \rm There are two types of extensions: the left extension and the right extension. Let $r$ $=$ $X$ $\rightarrow$ $Y$ be an IR and $E$ be an interval. The left extension is that the interval $E$ is inserted into the antecedent of $r$. Meanwhile, $E$ can form particular relations with the intervals in $r$. In other words, if we extend the same type of interval but with different relations to IR, we get a different new IR. It is similar to the I- and S- extensions in the point-based sequence \cite{gan2020fast, wang2016efficiently, gan2020proum, yin2012uspan}. The right extension is to insert $E$ into the consequent of $r$.
\end{definition}

\begin{figure}[h]
    \centering
    \includegraphics[width=0.94\linewidth]{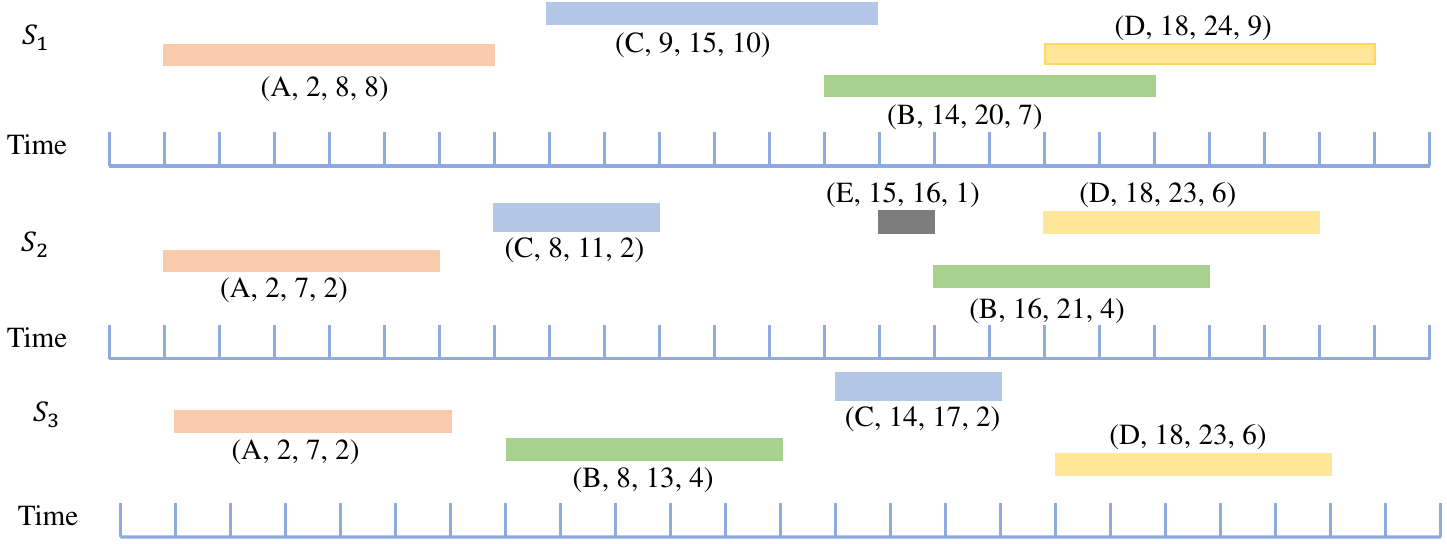}
    \caption{A running example interval database.}
    \label{example}
\end{figure}

\begin{table}[h]
	\centering
	\caption{The discovery UIRs with \textit{minutil} = 35 and \textit{minconf} = 0.6.}
	\label{UIRs}
     \scalebox{0.97}{
        \begin{tabular}{|c|c|c|c|c|}  
    		\hline 
    		\textbf{UIR} & \textbf{utility} & \textbf{conf} & \textbf{relations} & \textbf{numerical relations}\\
    		\hline 
    		\{A\} $\rightarrow$ \{B, D\} & 36 & 0.67 & \{\textbf{b}, \textbf{b}, \textbf{o}\} & ($0_{7}$, $2{_7}$)\\ 
    		\hline 
    		\{A\} $\rightarrow$ \{C, D\} & 47 & 1 & \{\textbf{b}, \textbf{b}, \textbf{b}\} & ($0_{7}$, $0{_7}$)\\ 
    		\hline 
    		\{A, B\} $\rightarrow$ \{D\} & 36 & 0.67 & \{\textbf{b}, \textbf{b}, \textbf{o}\} & ($0_{7}$, $2{_7}$)\\ 
    		\hline 
                \{A, C\} $\rightarrow$ \{D\} & 47 & 1 & \{\textbf{b}, \textbf{b}, \textbf{b}\} & ($0_{7}$, $0{_7}$)\\ 
    		\hline 
                \{C\} $\rightarrow$ \{D\} & 35 & 1 & \{\textbf{b}\} & ($0{_7}$)\\ 
    		\hline 
    \end{tabular}
     }	
\end{table}

Fig. \ref{example} is a running example database used in this paper. Table \ref{UIRs} is the output UIRs with \textit{minutil} = 35 and \textit{minconf} = 0.6. In this interval database, there are three ordered E-sequences: $S_1$, $S_2$, and $S_3$. The size of $S_1$ is 4 since there are 4 intervals appearing in $S_1$. Besides, there is an IR $r_1$ = $\{$\textit{A}$\}$ $\rightarrow$ $\{\textit{B}$, \textit{D}$\}$ with relations $\mathcal{R}$ $=$ \{\textbf{b}, \textbf{b}, \textbf{o}\}, i.e., \textit{A} \textbf{b} \textit{B}, \textit{A} \textbf{b} \textit{D}, and \textit{B} \textbf{o} \textit{D}. Since IR $r_1$ appears in two E-sequences: $S_1$ and $S_2$, the support of $r_1$ is equivalent to $\lvert$$seq(r_1)$$\lvert$ / $\lvert$$\mathcal{D}$$\rvert$ = $\lvert$\{$S_1$, $S_2$\}$\lvert$ / $\lvert$\{$S_1$, $S_2$, $S_3$\}$\rvert$ = 0.67. The antecedent occurs in three E-sequences: $S_1$, $S_2$, and $S_3$, so the support of $r_1$'s antecedent is $\lvert$\{$S_1$, $S_2$, $S_3$\}$\rvert$ / $\lvert$\{$S_1$, $S_2$, $S_3$\}$\rvert$ = 1. Thus, the confidence of $r_1$ is 0.67 / 1 = 0.67. Also, the utility of $r_1$ is $u$($r$) $=$ $\sum_{S_k \in \textit{seq}(r_1) \land \textit{seq}(r_1) \subseteq \mathcal{D}}$ $u$($r_1$, $S_k)$ = $\sum_{E \in r_1 \land S_1 \subseteq \textit{seq}(r_1)}$ $E.u$ $+$ $\sum_{E \in r_1 \land S_2 \subseteq \textit{seq}(r_1)}$ $E.u$ = $24$ + $12$ = 36. $r_1$ is generated from the candidate IR $r$ $=$ $\{$$A$$\}$ $\rightarrow$ $\{$$D$$\}$ by performing a left extension to insert an interval $B$ into the $r$'s antecedent with relations $\mathcal{R}$ $=$ \{\textbf{b}, \textbf{o}\}, i.e., $A$ \textbf{b} $B$ and $B$ \textbf{o} $D$. $r$ also can insert $B$ into its antecedent with relations $A$ \textbf{b} $B$ and $B$ \textbf{b} $D$ to generate IR $r^\prime$ $=$ $\{$\textit{A}, \textit{B}$\}$ $\rightarrow$ $\{$\textit{D}$\}$ with relations $\mathcal{R}$ $=$ \{\textbf{b}, \textbf{b}, \textbf{b}\}, but the confidence of $r^\prime$ is 0.33. Thus, $r^\prime$ is not a UIR.
\section{Proposed Method}   \label{sec:method}

In this section, we introduce relation representation, some critical pruning strategies, and the data structures used in this paper. Then, we propose a novel algorithm, namely UIRMiner, that can efficiently discover all utility-driven interval rules.

\subsection{Relation representation}

In general, we need to use a matrix to represent all relations in an E-sequence. For example, Fig. \ref{RR} shows the relations of $S_1$. Fig. \ref{RR}(a) is the original relations $\mathcal{R}$ of $S_1$. Note that each value in the matrix corresponds to a value in $\mathcal{R}$ when arranged from top to bottom, left to right. From Fig. \ref{RR}(a), we can observe that there are many \textbf{b} in $S_1$. This is because, in an ordered E-sequence, those intervals at the end of the E-sequence tend to only form a relation \textbf{b} with those intervals at the head of the E-sequence.

In STIPA \cite{sharma2018stipa}, they used a compressed relation representation like Fig. \ref{RR}(b). However, it still needs to know which relation should be recorded. The Z-miner algorithm \cite{lee2020z} introduced a quick relation confirmation technique. They divided a candidate E-sequence into two parts: \textit{earlier intervals} and \textit{checking area}. The intervals that require the confirmation of relations are located in \textit{checking area}. Although it can avoid many unnecessary comparisons, Z-miner still needs to construct \textit{earlier intervals} and \textit{check area} each time.

Inspired by \cite{sharma2018stipa} and \cite{lee2020z}, we propose a new relation representation, as a numerical encoding relation representation, to store the relations between intervals. Since there are only 7 relations that can be formed between two intervals, we use a 7-ary-based numerical encoding to represent the relations between intervals. A 7-ary number means that each digit in the number is between 0 and 6. An octal-based numerical encoding can also represent the relations, but there is a redundant value. 7-ary-based numerical encoding can use less than three bits to represent a digit, which can save more storage space. For example, as Fig. \ref{RR}(c) shows, each interval's relation is represented by a 7-ary number. Each column in Fig. \ref{RR} (a) corresponds to a number in Fig. \ref{RR} (c). As an illustration, consider the third column of Fig. \ref{RR} (a). The top ``b" corresponds to the highest digit ``0" in Fig. \ref{RR} (c), while the second ``b" corresponds to the middle digit ``0" in Fig. \ref{RR} (c). Finally, the ``o" corresponds to the last digit ``2" in Fig. \ref{RR} (c).

\begin{figure}[h]
    \centering
    \includegraphics[width=0.85\linewidth]{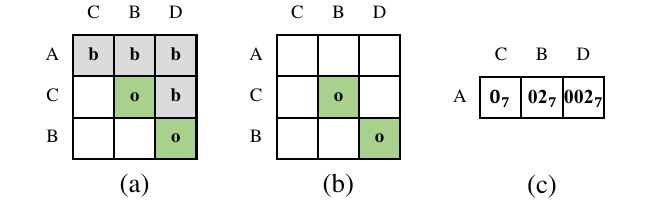}
    \caption{Different relation representations of $S_1$.}
    \label{RR}
\end{figure}

For the interval \textit{B} (the third interval in $S_1$), its relation is $02_7$. It potentially divides the interval relation into two parts: the zeros in front of the number and the non-zeros in the back. These two parts correspond to the \textit{earlier intervals} and \textit{checking areas}, respectively. Note that each digit of $02_7$ represents a relation between the intervals before \textit{B} and \textit{B}. For instance, the first digit $0$ represents the relation \textbf{b} between intervals \textit{A} and \textit{B}, and the second digit symbolizes the relation \textbf{o} between intervals \textit{C} and \textit{B}, (i.e., \textit{A} \textbf{b} \textit{B} and \textit{C} \textbf{o} \textit{B}). Therefore, when we determine the relation of \textit{D} we only need to check the relation of the intervals corresponding to the non-zero part in the previous interval relation, i.e., interval \textit{B}'s relation.

With the help of numerical encoding relation representation, we can quickly determine all relations of intervals to others, avoiding many unnecessary comparisons. We do not construct any preliminary data structures to maintain them. 

\subsection{Pruning strategies}

In utility-driven knowledge discovery, the utility upper bound is the general approach to pruning the search space. Thus, in this part, we will present several widely used utility upper bounds. Then, we combine the utility upper bound and relations to design a novel utility pruning strategy. In point-based sequential rule mining \cite{huang2021us, zhang2022totally, zida2015efficient}, there are many pruning strategies related to utility, such as \textit{SEU}, \textit{LERSPEU}, and \textit{RERSPEU}. Note that the confidence pruning strategy \cite{zhang2022totally} is frequency-dependent. We first introduce these strategies, and then present the proposed pruning strategy in light of the relation and utility upper bounds.

\begin{upper bound}[Sequence estimated utility of interval]
    \label{SEU}
    \rm Let $E$ be an interval and $\mathcal{D}$ be an interval database. The sequence estimated utility of $E$, denoted as \textit{SEU}($E$), is defined as \textit{SEU}($E$) $=$ $\sum_{E \in S_k \land S_k \in seq(E)}$\textit{E.u} \cite{zida2015efficient}.
\end{upper bound}

\begin{upper bound}[Left extension reduced sequence prefix extension utility]
    \label{Upper bound: LERSPEU}
    \rm Given a candidate IR $c$, an E-sequence $S$, and the other IR $r$, $r$ is extended from $c$ by performing a left extension with an interval $E$. The left extension reduced sequence prefix extension utility of $r$ in E-sequence $S$, denoted as \textit{LERSPEU}($r$, $S$), is defined as:
	$$ \textit{LERSPEU}(r, S)=\left\{
	\begin{gathered}
	u(r, S) + \textit{ELeft}(r, E, S), S  \in seq(r) \\
	0, otherwise.
	\end{gathered}
	\right.
	$$
\end{upper bound}

Let \textit{ELeft}(\textit{r, E, S}) represent the sum utility of the intervals from $E$ to the last interval in $S$ that can extend into the antecedent \cite{zhang2022totally}. The \textit{LERSPEU} of an IR $r$ in interval database $\mathcal{D}$, denoted as \textit{LERSPEU}(\textit{r}), can be defined as:
\begin{center}
	\textit{LERSPEU}($r$) $=$ $\sum_{S \in seq(r) \land seq(r) \subseteq \mathcal{D}}$ \textit{LERSPEU}($r$, $S$).
\end{center}

\begin{theorem}
    \label{Theorem:LERSPEU}
    Given an IR $r$ $=$ $X$ $\rightarrow$ $Y$ and the other IR $r^\prime$, where $r^\prime$ is extended with an interval $E$ from $r$ by performing a left extension, we have $u$($r^\prime$) $\le$ \textit{LERSPEU}($r^\prime$). 
\end{theorem}
\begin{proof}
    \label{Proof:LERSPEU}
    \rm Let $S$ be an E-sequence, $r$ be an IR, and $E$ be an interval that will be extended into the antecedent of $r$ to form IR $r^\prime$. According to \textit{ELeft}($r$, $E$, $S$) in upper bound \textit{LERSPEU}, we have $u$($E$, $S$) $\le$ \textit{ELeft}($r$, $E$, $S$). Then, we have $u$($r^\prime$, $S$) $=$ $u$($r$, $S$) $+$ $u$($E$, $S$) $\le$ $u$($r$, $S$) $+$ \textit{ELeft}($r$, $E$, $S$) $=$ \textit{LERSPEU}($r^\prime$, $S$). Therefore, $u$($r^\prime$) $\le$ \textit{LERSPEU}($r^\prime$).
\end{proof}

For example, there is a candidate interval rule $r$ = $\{A\}$ $\rightarrow$ $\{D\}$ with $\mathcal{R}$ $=$ \{\textbf{b}\} and an E-sequence $S_3$. If we want to extend an interval $C$ into its antecedent, the \textit{LERSPEU}($r$, $S_3$) = $u$($r$, $S_3$) + \textit{ELeft}($r$, $C$, $S_3$) = 8 $+$ 2 = 10. Note that the utility of interval $B$ is not included since $B$ is before $C$.

\begin{upper bound}[Right extension reduced sequence prefix extension utility]
    \label{Upper bound: RERSPEU}
    \rm Given a candidate IR $c$, an E-sequence $S$, and the other IR $r$, $r$ is extended from $c$ by performing a right extension with an interval $E$. The right extension reduced sequence prefix extension utility of $r$ in E-sequence $S$, denoted as \textit{RERSPEU}($r$, $S$), is defined as:
	$$ \textit{RERSPEU}(r, S)=\left\{
	\begin{gathered}
	u(r,S) + \textit{ERight}(r,E,S), S \in seq(r) \\
	0, otherwise.
	\end{gathered}
	\right.
	$$
\end{upper bound}

Here \textit{ERight}(\textit{r, E, S}) represents the sum utility of the intervals from $E$ to the last interval in $S$ that can be extended into the consequent \cite{zhang2022totally}. Thus, the \textit{RERSPEU} of an IR $r$ in a database $\mathcal{D}$, denoted as \textit{RERSPEU}(\textit{r}), can be defined as:
\begin{center}
    \textit{RERSPEU}($r$) $=$ $\sum_{S \in seq(r) \land seq(r) \subseteq \mathcal{D}}$ \textit{RERSPEU}($r$, $S$).
\end{center}

\begin{theorem}
    \label{Theorem:RERSPEU}
    Given an IR $r$ $=$ $X$ $\rightarrow$ $Y$ and the other IR $r^\prime$, where $r^\prime$ is extended with an interval $E$ from $r$ by performing a right extension, we have $u$($r^\prime$) $\le$ \textit{RERSPEU}($r^\prime$). 
\end{theorem}
\begin{proof}
    \label{Proof:RERSPEU}
    \rm Let $S$ be an E-sequence, $r$ be an IR, and $E$ be an interval that can be extended into the consequent of $r$ to form an IR $r^\prime$. According to \textit{ERight}($r$, $E$, $S$) in upper bound \textit{RERSPEU}, we have $u$($E$, $S$) $\le$ \textit{ERight}($r$, $E$, $S$). Then, we have $u$($r^\prime$, $S$) $=$ $u$($r$, $S$) $+$ $u$($E$, $S$) $\le$ $u$($r$, $S$) $+$ \textit{ERight}($r$, $E$, $S$) $=$ \textit{RERSPEU}($r^\prime$, $S$). Thus, $u$($r^\prime$) $\le$ \textit{RERSPEU}($r^\prime$).
\end{proof}

For example, there is a candidate interval rule $r$ = $\{A\}$ $\rightarrow$ $\{C\}$ with $\mathcal{R}$ $=$ \{\textbf{b}\} and an E-sequence $S_1$. If we want to extend an interval $D$ into its antecedent, the \textit{LERSPEU}($r$, $S_1$) = $u$($r$, $S_1$) + \textit{ERight}($r$, $C$, $S_1$) = 18 $+$ 9 = 27. Note that the utility of interval $B$ is not included since $B$ is before $D$.

\begin{strategy}[Unpromising interval pruning strategy]
    \label{Strategy:UIP}
    \rm Given an interval database $\mathcal{D}$, UIRMiner will remove all unpromising intervals from $\mathcal{D}$. For an unpromising interval $E$, $ E$'s \textit{SEU} is smaller than \textit{minutil}. As a result, the utility of any IR that contains an interval $E$ will not exceed the \textit{minutil}. In other words, the unpromising interval $E$ will not be contained in an IR, which means the interval $E$ is useless for UIRs. Thus, we can remove the interval from $\mathcal{D}$ directly. Similar to US-Rule \cite{huang2021us} and TotalSR \cite{zhang2022totally}, after removing some intervals from $\mathcal{D}$, the other intervals' \textit{SEU} will be changed. UIRMiner will keep using the unpromising interval pruning strategy until no intervals are removed.
\end{strategy}

\begin{strategy}[Left expansion reduced sequence prefix extension utility pruning strategy]
    \label{Strategy:LERSPEU}
    \rm Given an IR $r$, according to Theorem \ref{Theorem:LERSPEU}, when $r$ implements a left expansion with a specific interval $E$ to generate the other IR $r^\prime$, the utility of $r^\prime$ will not exceed the upper bound \textit{LERSPEU}, i.e., $u$($r^\prime$) $\le$ \textit{LERSPEU}($r^\prime$) \cite{zhang2022totally}. If \textit{LERSPEU}($r^\prime$) $\textless$ \textit{minutil}, we can know that any IR extending from $r$ by performing a left expansion will not be a UIR. Thus, we can stop extending further.
\end{strategy}

\begin{strategy}[Right expansion reduced sequence prefix extension utility pruning strategy]
    \label{Strategy:RERSPEU}
    \rm Given an IR $r$, according to Theorem \ref{Theorem:RERSPEU}, when $r$ implements a right expansion with a specific interval $E$ to generate the other IR $r^\prime$, the utility of $r^\prime$ will not exceed the upper bound \textit{RERSPEU}, i.e., $u$($r^\prime$) $\le$ \textit{RERSPEU}($r^\prime$) \cite{zhang2022totally}. If \textit{RERSPEU}($r^\prime$) $\textless$ \textit{minutil}, we can know that any IR extending from $r$ by performing a right expansion will not be a UIR. Thus, we can stop extending further.
\end{strategy}

The relations between the extending interval and the intervals in the candidate IR are unknown when computing the upper bounds. We propose a \textit{utility complement pruning} strategy based on \textit{LERSPEU} and \textit{RERSPEU}. The idea of the \textit{utility complement pruning} strategy is that for each extending interval and the last interval in the candidate IR, there are seven relations that can be formed in theory. For each exact relation, the upper bound of utility is mutually exclusive.

\begin{strategy}[utility complement pruning strategy]
    \label{Strategy:RERSPEU}
    \rm Given an IR $r$, an interval $E$, and corresponding utility upper bound $U$, we extend $E$ to IR $r$ with relation order \textit{0, 1, 2, 3, 4, 5, 6} as shown in Fig. \ref{relation}. Note that the relationship is between $E$ and the last interval in $r$'s antecedent or consequent according to the left extension or right extension. Therefore, when performing some relation extensions, we can know the remaining relation's utility upper bound. If the remaining relation's utility upper bound is less than \textit{minutil}, we can stop extending the remaining relations and break out.
\end{strategy}

For example, in Table \ref{UIRs}, there is a UIR $r_1$ $=$ $\{A\}$ $\rightarrow$ $\{B, D\}$ with $\mathcal{R}$ $=$ \{\textbf{b}, \textbf{b}, \textbf{o}\}. $r$ is extended from a candidate IR $r^\prime$ $=$ $\{A\}$ $\rightarrow$ $\{B\}$ with $\mathcal{R}$ $=$ \{\textbf{b}, \textbf{b}\} and \textit{RERSPEU}($r^\prime$) $=$ $\sum_{S \in seq(r^\prime) \land seq(r^\prime) \subseteq \mathcal{D}}$\{$u$($r^\prime$, $S$) $+$ $u$($E$, $S$)\} = 15 $+$ 9 $+$ 6 $+$ 6 $+$ 6 $+$ 6 $=$ 49. Thus, when $r^\prime$ extends an interval $D$ into its consequent, we know that $B$ and $D$ can form seven relations. We first extend $D$ with the relation R($B$, $D$) = \textbf{b}. According to the database, we know that \textit{RERSPEU}($r^\prime$) under this condition is 12, which is smaller than the \textit{minutil}. We can subtract 12 from 49 to get a lower \textit{RERSPEU}($r^\prime$). When extending $D$ with the other relations \textit{RERSPEU}($r^\prime$) = 49 $-$ 12 = 37. After extending $D$ with the relation R($B$, $D$) = \textbf{o}, \textit{RERSPEU}($r^\prime$) becomes 0. Therefore, we can stop extending $D$ directly.

Fig. \ref{CP} vividly shows the general situations. Note that the total upper bound can represent either \textit{LERSPEU} or \textit{RERSPEU}. Each block signifies one specific extending relation and its utility upper bound. In Fig. \ref{CP}(a), we know that when we extend an interval with a relation \textbf{b} we can get a promising extension. Then we can extend further. After that, we can subtract the already-used relation utility upper bound from the total upper bound. Therefore, like Fig. \ref{CP}(b) when we perform the other specific relation extension, we can know in advance that they are unpromising extensions and break out directly.

\begin{figure}[h]
    \centering
    \includegraphics[width=0.9\linewidth]{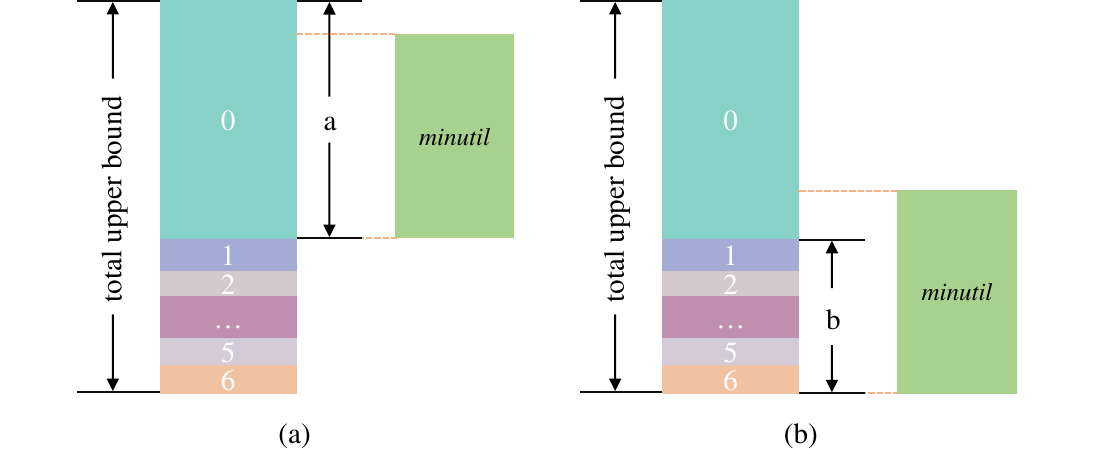}
    \caption{Situations of complement utility upper bound.}
    \label{CP}
\end{figure}

The above pruning strategies are related to utility, but in the area of rule mining, confidence is related to support. As a result, we can design a pruning strategy that takes support into account. Since UIRMiner stipulates that it will extend the right part after extending the left part, we can use the confidence pruning strategy \cite{zhang2022totally}. Because when a right extension is performed, the support of the antecedent is fixed. The support for IR will decrease due to this anti-monotonic property. Thus, we can take advantage of this property to prune the search space.

\begin{strategy}[Confidence pruning strategy]
    \label{Strategy:CPS}
    \rm Given an IR $r$, when $r$ implements a right expansion with a specific interval $E$ to generate the other IR $r^\prime$, the confidence value of $r^\prime$ is less than or equal to $r$'s confidence value, i.e., \textit{conf}($r^\prime$) $\le$ \textit{conf}($r$). If \textit{conf}($r$) $\textless$ \textit{minconf}, we have \textit{conf}($r^\prime$) $\textless$ \textit{minconf} too. Thus, we can stop extending further \cite{zhang2022totally}.
\end{strategy}

\subsection{Data structures}

In this part, several data structures are introduced: the E-sequence array, the antecedent utility-list, and the consequent utility-list.

\begin{definition}[E-sequence array]
    \label{seqArray}
    \rm Let $S$ = $\{$$E_1$, $E_2$, $\cdots$, $E_n$$\}$ be an E-sequence with size $n$. E-sequence array is composed of six sub-arrays, intervals, \textit{st}, \textit{ft}, utility, remaining utility (\textit{ru}), and the same start time position (\textit{sstp}). The intervals array saves all intervals that appear in $S$. The \textit{st} array records the corresponding start time of each interval. The \textit{ft} array records the finish time of each interval. The utility array contains information about the interval's utility. The remaining utility array records the sum of utilities after this interval. At last, the \textit{sstp} array keeps the position of the first same start time interval.
\end{definition}

\begin{figure}[h]
    \centering
    \includegraphics[width=0.9\linewidth]{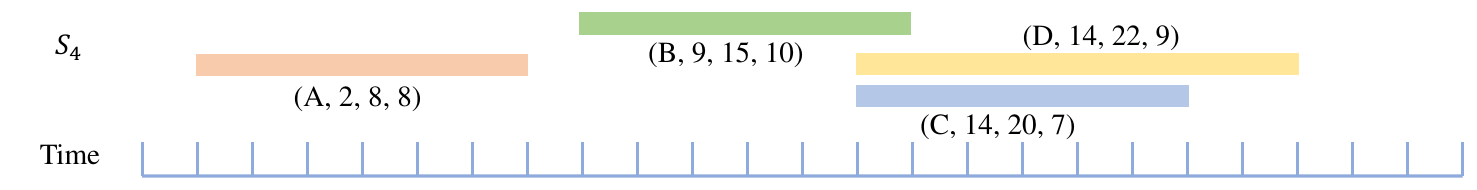}
    \caption{An illustration of the function of \textit{sstp}.}
    \label{sstp}
\end{figure}

Table \ref{seqArray} shows an E-sequence array of $S_4$ illustrated in Fig. \ref{sstp}. Since intervals $C$ and $D$ share the same start time, $D$'s \textit{sstp} is the position of interval $C$. The purpose of \textit{sstp} is to acquire the extendable intervals. Assuming there is a candidate IR $r$ = \{$A$\} $\rightarrow$ \{$D$\} with $\mathcal{R}$ $=$ \{\textbf{b}\}. The extendable interval of the left extension of $r$ is only $B$ excluding $C$ even though $C$'s position is less than $D$. Intervals $C$ and $D$ occur simultaneously, so we cannot use $C$ to predict $D$. This is why we can't extend $C$ and introduce \textit{sstp}.

\begin{table}[h]
	\centering
	\caption{\textit{E-sequence array} of $S_4$}
	\label{seqArray}
	\begin{tabular}{|c|c|c|c|c|c|}  
		\hline 
		\textbf{Interval} & \textbf{st} & \textbf{ft} & \textbf{utility} & \textbf{ru}& \textbf{sstp}\\
		\hline 
		\(A\) & 2 & 8 & 8 & 26 & 1\\ 
		\hline 
		\(B\) & 9 & 15 & 10 & 16 & 2\\
		\hline 
		\(C\) & 14 & 20 & 7 & 9 & 3\\
		\hline 
            \(D\) & 14 & 22 & 9 & 0 & 3\\
		\hline 
	\end{tabular}
\end{table}

\begin{definition}[antecedent utility-list]
    \label{AUL}
    \rm Let $r$ $=$ $X$ $\rightarrow$ $Y$ be an IR. We use \textit{AUL} to denote the antecedent utility-list of $r$. The \textit{AUL} is a triple with elements: E-sequence identifier (Esid), interval rule occurring (IRO), and the upper bound on utility (UB). Since we still need to perform the right extension, the upper bound on utility not only records \textit{LERSPEU} but also records \textit{RERSPEU}. Furthermore, if the IRO domain is false, UB is set to zero.
\end{definition}

Since there are some E-sequences that do not contain the consequent of an IR, \textit{AUL} will record all E-sequences that the antecedent occurs for calculating the support of the antecedent correctly. IRO is a boolean value used to mark E-sequences that contain the entire interval rule instead of just the antecedent. With the help of IRO we can conveniently calculate the support of the rule. In addition, the right extension follows the left extension, so the utility upper bound can record the sum of \textit{LERSPEU} and \textit{RERSPEU}. 

\begin{definition}[consequent utility-list]
    \label{CUL}
    \rm Let $r$ $=$ $X$ $\rightarrow$ $Y$ be an IR. We use \textit{CUL} to denote the consequent utility-list of $r$. The \textit{CUL} is a tuple with elements: an E-sequence identifier (Esid) and a utility upper bound (UB). In \textit{CUL}, UB only records \textit{RERSPEU}.
\end{definition}

\begin{table}[h]
	\centering
	\caption{\textit{AUL} of candidate interval rule $r_1$ in Table \ref{UIRs}}
	\label{AUL}
	\begin{tabular}{|c|c|}  
		\hline 
		\textbf{AUL} & $<$($S_1$, true, 34), ($S_2$, true, 14), ($S_3$, false, 0)$>$\\
		\hline 
	\end{tabular}
\end{table}

Table \ref{AUL} shows the \textit{AUL} of UIR $r_1$ in Table \ref{UIRs}. We can know that $r_1$ appears in two E-sequences, $S_1$ and $S_2$, but not in $S_3$. Thus, the IRO in the third triple is false. We can traverse \textit{AUL} to calculate the confidence of an IR and its UB for further extension. The structure of \textit{CUL} is similar to \textit{AUL}, and we will not show it in detail.

\subsection{The UIRMiner algorithm}

\begin{algorithm}[h]
    \small
    \caption{UIRMiner}
    \label{alg:IRMiner}
    \KwIn{$\mathcal{D}$: an interval database, \textit{minutil}: the minimum utility threshold, \textit{minconf}: the minimum confidence threshold.}
    \KwOut{UIRs: the set of all utility-driven interval rules.}
	
    initialize $\Sigma$ $\leftarrow$ $\emptyset$;
	
    calculate \textit{SEU}($E$) for each $E$ in $\mathcal{D}$ and update $\Sigma$;
	
    \While {\rm$\exists$ $E$ $\in$ $\Sigma$ and \textit{SEU}($E$) $\textless$ \textit{minutil}} {
	remove all unpromising intervals in $\Sigma$ and recalculate \textit{SEU};
    }
	
    scan $\mathcal{D}$ to construct E-seq-array, identify the set of antecedent with size 1: \textit{A};
	
    \For {$a$ $\in$ \textit{A}}{
	identify the set of intervals that can be consequent: \textit{C};\\
		\For {$c$ $\in$ \textit{C}} {
			generate IR $r$, compute \textit{u}($r$) and \textit{conf}($r$);\\
			construct \textit{AUL}($r$) and \textit{CUL}($r$);\\
			\If {\rm$u$($r$) $\ge$ \textit{minutil} and \textit{conf}($r$) $\ge$ \textit{minconf}}{
				update UIRs$\leftarrow$ UIRs$\cup$ $r$;
			}
			\If {\rm \textit{AUL($r$).UB} $\ge$ \textit{minutil}}{
				call \textbf{Expansion}($r$, \textit{AUL}($r$), true, 0);
			}
			\If{\rm\textit{conf}($r$) $\ge$ \textit{minconf} and \textit{CUL.UB}($r$) $\ge$ \textit{minutil}}{
		   		 call \textbf{Expansion}($r$, \textit{CUL}($r$), false, \textit{AUL}.\textit{length});
			}
		}
	}	
    \Return{\textit{UIRs}}\;
\end{algorithm}

\begin{algorithm}[h]
    \small
    \caption{Extension}
    \label{alg:leftExpansion}
    \KwIn{$r$: an IR, \textit{UL}($r$): AUL or CUL of $r$, flag: left extension flag, sup: the support of antecedent.}

    identify all intervals \textit{I} that can be extended into antecedent or consequent according to the flag;
    \For{$E$ $\in$ \textit{I}}{
        calculate complement utility upper bound \textit{U};\\
        \While{\rm\textit{U} $\ge$ \textit{minutil}} {
            \For{\rm $i$ $\in$ $0\sim6$} {
             calculate relation \textit{i}'s utility upper bound \textit{RUB};\\
                    \If{\textit{RUB} $\ge$ \textit{minutil}} {
                        extend $r$ with $E$ according to flag to form candidate IR set $R$;\\
                        \For{\rm$r^\prime$ $\in$ $R$} {
                            calculate its utility $u$($r^\prime$) and \textit{conf}($r^\prime$), construct its \textit{CUL};\\
                            \If{flag} {
                               construct its \textit{AUL}; \\
                            }
                            \If {\rm$u$($r^\prime$) $\ge$ \textit{minutil} and \textit{conf}($r^\prime$) $\ge$ \textit{minconf}}{
				                update UIRs$\leftarrow$ UIRs$\cup$ $r^\prime$;
			                 }
                            \If{flag and \rm \textit{AUL}($r^\prime$).\textit{UB} $\ge$ \textit{minutil}} {
                                call \textbf{Expansion}($r^\prime$, \textit{AUL}($r^\prime$), true, 0);
                            }
                            \If{\rm\textit{conf}($r^\prime$) $\ge$ \textit{minconf} and \textit{CUL}($r^\prime$).\textit{UB} $\ge$ \textit{minutil}} {
                                call \textbf{Expansion}($r^\prime$, \textit{CUL}($r^\prime$), false, \textit{AUL}.\textit{length});
                            }
                        }
                    } 
                    \textit{U} = \textit{U} - \textit{RUB};\\
                    \If{\textit{U} $\textless$ \textit{minutil}} {
                        break;
                    }
                }
            }
        }
\end{algorithm}

In Algorithm \ref{alg:IRMiner}, UIRMiner takes an interval database, the minimum utility threshold, and the minimum confidence threshold as its input. It outputs all utility-driven interval rules. UIRMiner initializes the $\Sigma$, computes \textit{SEU} of each interval, and removes the unpromising intervals (Lines 1-5). After that, UIRMiner constructs the E-seq array for each E-sequence and generates all antecedent intervals \textit{A} with size 1 (Line 6). Then, for each antecedent in \textit{A}, UIRMiner identifies the corresponding consequent interval set \textit{C} with size 1 (Lines 7-8). For each consequent in \textit{C}, the IR $r$ with size 1 $\ast$ 1 is generated. Then the corresponding data structures \textit{AUL}($r$) and \textit{CUL}($r$) are constructed (Lines 9-11). Subsequently, if the utility and the confidence of $r$ are both above the thresholds, UIRMiner will update the output rule set (Lines 12-14). If the \textit{AUL}($r$).\textit{UB} satisfies the \textit{minutil} condition, UIRMiner calls an Extension to extend the antecedent of $r$ (Lines 15-17). If both the confidence and \textit{CUL}($r$) satisfy \textit{minconf} and \textit{minutil} respectively, UIRMiner calls an extension to extend the consequent of $r$. Note that when we perform the right extension, the \textit{confidence pruning} strategy can be used (Lines 18-20).

Algorithm \ref{alg:leftExpansion} presents the extension of UIRMiner. It takes the candidate IR $r$ and its utility-list data structure \textit{UL}($r$) as a flag to represent whether a left extension is performed, and the antecedent support as the input. We should use the flag to signify the left extension since the right extension is always performed. Firstly, it confirms the interval set \textit{I} containing the intervals that can be extended into $r$ (Line 1). For each interval $E$ in \textit{I}, Extension calculates the complement utility upper bound \textit{U} according to the data structure of $r$ (Lines 2-3). The type of data structure depends on the value of the flag to specify the left extension or right extension. \textit{U} is the naive utility upper bound that does not consider the relation between the extending intervals $E$ and $r$. While \textit{U} is greater than or equal to \textit{minutil}, Extension extends $E$ in the relation order of 0 $\sim$ 6. Note that the relation is between $E$ and the last interval in the antecedent. Then, for a particular relation \textit{i}, leftExtension calculates its simple relation utility upper bound \textit{RUB} (Lines 4-6). If \textit{RUB} satisfies the minimum utility condition, Extension generates all candidate IRs, where the relation of each IR has been confirmed (Lines 7-8). For each specific IR $r^\prime$, the Extension procedure computes its utility and confidence and constructs its consequent utility-list, according to the value of the flag to construct the antecedent utility-list (Lines 9-13). If the utility and confidence of $r^\prime$ satisfy the condition simultaneously, leftExtension outputs $r^\prime$ (Lines 14-16). If the flag is true and \textit{AUL}($r^\prime$).\textit{UB} is not less than \textit{minutil}, the Extension is called for extending the antecedent further (Lines 17-19). If the confidence of $r^\prime$ is above \textit{minconf} and \textit{CUL}($r^\prime$).\textit{UB} exceeds \textit{minutil}, Extension calls for extending the consequent further with parameter flag: false (Lines 20-22). Then, Extension updates \textit{U}, and if \textit{U} is less than \textit{minutil}, it breaks out directly to avoid unnecessary extending (Lines 25-28). There are two reasons why extend an interval in the relation order of 0 $\sim$ 6. Firstly, the longer the IR, the more complex IR relations will be generated. We do not know the exact number of IRs when extending an interval. If we only consider the relation between the extending interval and the last interval in the antecedent or consequent, we can divide the IRs into 7 classes. Secondly, after we divide IRs into seven classes, the utility complement pruning strategy can be utilized to shrink the search space. We can filter many unpromising UIRs under the rough relation.

\subsection{Complexity analysis}

\textbf{Time complexity:} Let $|$D$|$ be the number of E-sequences in the database $\mathcal{D}$ and \textit{L} be the longest size of E-sequence in $\mathcal{D}$. The database scanning (Algorithm \ref{alg:IRMiner} lines 1-2) time complexity is $O$($|$D$|$$L$). Assume that there are $k$ hopeless intervals in the database and that the database is scanned at most $k$. Thus, the pre-processing time complexity is $O$($k$$|$D$|$$L$). Since the process of discovering the utility-driven interval rule is a deep search, the search space is a multi-fork tree, as shown in Fig. \ref{searchTree}. Therefore, the time complexity of the search stage is exponential. We can use \textit{O}($x^{h}$) to roughly represent it, where \textit{x} is greater than or equal to 1 and \textit{h} is the depth of search space. Usually, \textit{x} is greater than 1. When we do not use pruning strategies, completing the algorithm will take a lot of time. The function of pruning strategies is to try to reduce the value of \textit{x}. Although the time complexity is still exponential, the effect is tremendous. For relation calculation, the worst-case time complexity is $O$($L^2$) when there is little \textbf{b} relation between intervals. For constructing a data structure, there are most $|$D$|$ projected E-sequences, so its time complexity is $O$($|$D$|$). Finally, the total time complexity is $O$($k$$|$D$|$$L$ $+$ ($L^2$ $+$ $|$D$|$)$x^{h}$).

\textbf{Space complexity:} In many sequential rule mining algorithms \cite{huang2021us,zhang2022totally,zida2015efficient}, they prepare the sub-database of the next level-extending nodes, as shown in Fig. \ref{searchTree} (a). It will pose an exponential space complexity. In Fig. \ref{searchTree} (b), UIRMiner does not initialize the sub-database of the next-level nodes but identifies all extendable events. Instead, UIRMiner will construct the corresponding database when a node is extended. The space complexity of this mode will approximate linear complexity, i.e., \textit{O}(\textit{h}). The pruning strategies have little impact on the space complexity in this extending mode.

\begin{figure}[h]
    \centering
    \includegraphics[width=1\linewidth]{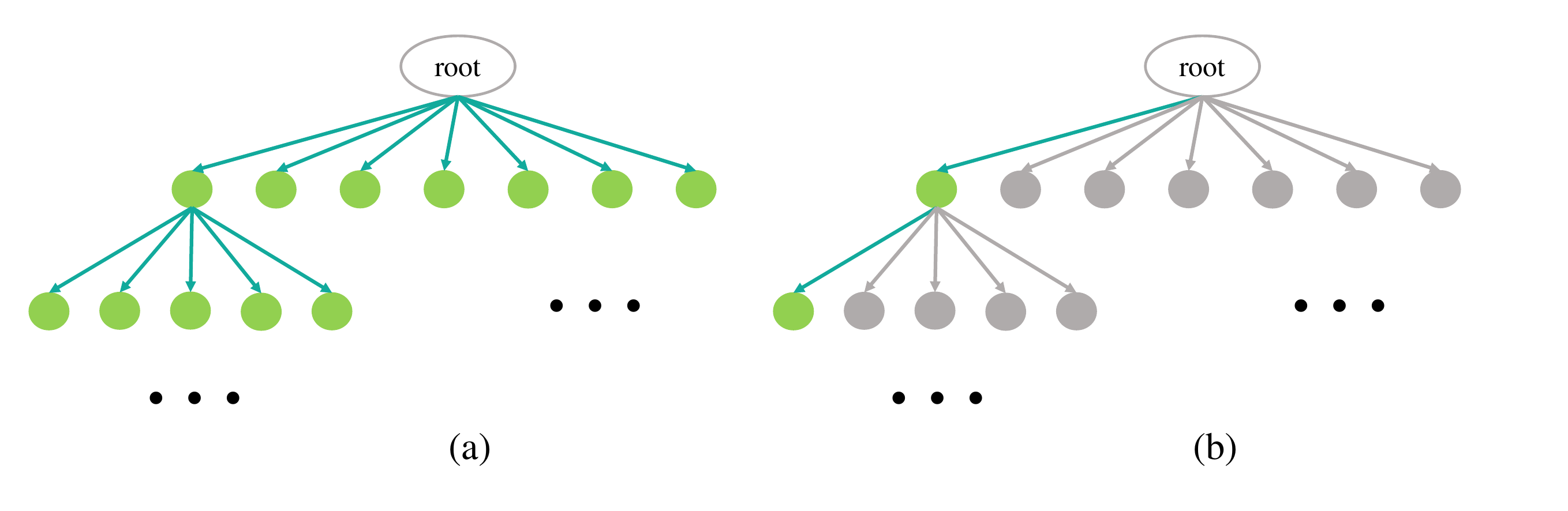}
    \caption{The general search space of UIRMiner.}
    \label{searchTree}
\end{figure}
\section{Experiments}   \label{sec:experiments}

In this section, we evaluate the effectiveness and efficiency of UIRMiner by conducting several experiments on both real-world and synthetic datasets. Pattern discovery has been widely used in many domains and applications \cite{fischer2020discovering,li2018truepie,preti2021maniacs}. To achieve this goal, we designed two variants of UIRMiner: UIRMiner$_{nsc}$ and UIRMiner$_{nc}$. 

\textbf{Parameter settings.} UIRMiner$_{nsc}$ uses all pruning strategies, but its relation to each IR is represented by a matrix, which means there is no storage compression for UIRMiner$_{nsc}$. UIRMiner$_{nc}$ is the algorithm that does not use the \textit{utility complement pruning} strategy to verify the effectiveness of this strategy. UIRMiner adopts all optimizations to demonstrate its performance. Furthermore, all experiments are conducted in Java and run on a personal computer with a Windows 10 system, an Intel Core i7-10700K processor, and 32 GB of main memory. The source code and datasets are available at GitHub https://github.com/DSI-Lab1/UIRMiner.

\begin{table}[htbp]
	\caption{Features of the datasets} 
	\label{datasets}
	\centering
	\begin{tabular}{|c|c|c|c|c|}
    	\hline
    	\textbf{Dataset} &  \textbf{\#E-seq}  &  \textbf{$\vert \textit{$\Sigma$} \vert$} & $\textit{avg}(\textit{e-Seq})$ & \textbf{\#Intervals}  \\
    	\hline 
            ASLBU1 & 873 & 216 & 16.95 & 14,802 \\ \hline
            ASLBU2 & 1,839 & 254 & 22.71 & 41,761 \\ \hline
            CONTEXT & 240 & 54 & 358.42 & 19,355 \\ \hline
            HEPATITIS & 498 & 63 & 855.89 & 53,921 \\ \hline
    	10k200E32L  & 10,000  & 200 & 31.49 & 314942 \\ \hline
            20k200E32L  & 20,000  & 200 & 31.47 & 629453\\ \hline
            40k200E32L  & 40,000  & 100 & 31.51 & 1260586\\ \hline
            50k200E32L  & 50,000  & 100 & 31.5 & 1574919\\ \hline
            60k200E32L  & 60,000  & 100 & 31.5 & 1889870\\ \hline
            70k200E32L  & 70,000  & 100 & 31.51 & 2205799\\ \hline
            80k200E32L  & 80,000  & 100 & 31.51 & 2520811\\ 
        \hline 
	\end{tabular}
\end{table}

\textbf{Datasets.} The details of the datasets used in this paper are listed in Table \ref{datasets}. The first four rows are real-world datasets, including ASLBU1, ASLBU2, CONTEXT, and HEPATITIS. ASLBU1 is generated from American Sign Language. ASLBU2 is its newer version. CONTEXT is the data from a mobile device of human activity. HEPATITIS is the dataset that records the E-sequences of patients suffering from either Hepatitis B or C. The utility of each interval in real-world datasets is generated from a Gaussian distribution like \cite{wang2016efficiently}, we also consider the interval's duration like the internal utility in \cite{yin2012uspan, wang2016efficiently, gan2020fast}. The remaining datasets are synthetic. In Table \ref{datasets}, \#E-seq signifies the number of E-sequences in the datasets, \textbf{$\vert$\textit{$\Sigma$}$\vert$} represents the number of different intervals of the dataset, $\textit{avg}(\textit{e-Seq})$ means the average intervals of each E-sequence, and \#Intervals denotes the total intervals in each dataset.

\textbf{Evaluation metrics.} For evaluating the proposed algorithm's performance, we will compare the well-known metrics: runtime consumption, the ability of pruning strategies, and memory usage. We also compare the storage consumption to verify the effectiveness of numerical encoding relation representation.

\subsection{Runtime evaluation}

The running time is a significant metric that evaluates the performance of the algorithm. To verify the efficiency of UIRMiner, in this part, we compare the running times of UIRMiner$_{nsc}$, UIRMiner$_{nc}$, and UIRMiner. Fig. \ref{runtime} illustrates the results of these three algorithms on both real-world and synthetic datasets. Particularly, we set the \textit{minconf} to 0.6 and the various \textit{minutil} to conduct the experiments. 

We can discover that in each dataset, UIRMiner takes the least time to complete the mining process, which means UIRMiner is the most efficient algorithm. For example, the result of the runtime on CONTEXT in Fig. \ref{runtime} (c) vividly shows that UIRMiner outperforms its other two versions. In Fig. \ref{runtime} (a), (b), and (d), although as the \textit{minutil} increases, the other two algorithms, UIRMiner$_{nsc}$ and UIRMiner$_{nc}$, can achieve similar results to UIRMiner. UIRMiner still consumes slightly less time than them. Especially in 10k200E32L and 20k200E32L, UIRMiner is obviously better than UIRMiner$_{nsc}$ and UIRMiner$_{nc}$. Since the number of E-sequences in the synthetic dataset is much larger than in the real-life dataset, the results are much more obvious.

To prove the utility complement pruning strategy is effective. We can compare the curves in each dataset of UIRMiner$_{nc}$ and UIRMiner. The most apparent comparison is in the synthetic datasets. In Fig. \ref{runtime} (e) and (f), we can observe that UIRMiner significantly outperforms UIRMiner$_{nc}$. In real-life datasets, the running time of UIRMiner is also better than UIRMiner$_{nc}$. Furthermore, from the results of UIRMiner$_{nsc}$ and UIRMiner, we can find out that the numerical encoding relation representation also can save time since the traditional relation determination needs to compute every two intervals' relations in an E-sequence, which is a $O$($n^2$) time complexity method. As the candidate IR grows, the later extended interval tends to form a \textbf{b} relation with the earlier intervals, which in numerical encoding relation representation does not take time to confirm. 

In summary, the abundant experiments shown above demonstrate that the proposed two optimizations, numerical encoding relation representation, and the utility complement pruning strategy can work well.

\begin{figure*}[htbp]
    \centering  
    \includegraphics[width=1\linewidth]{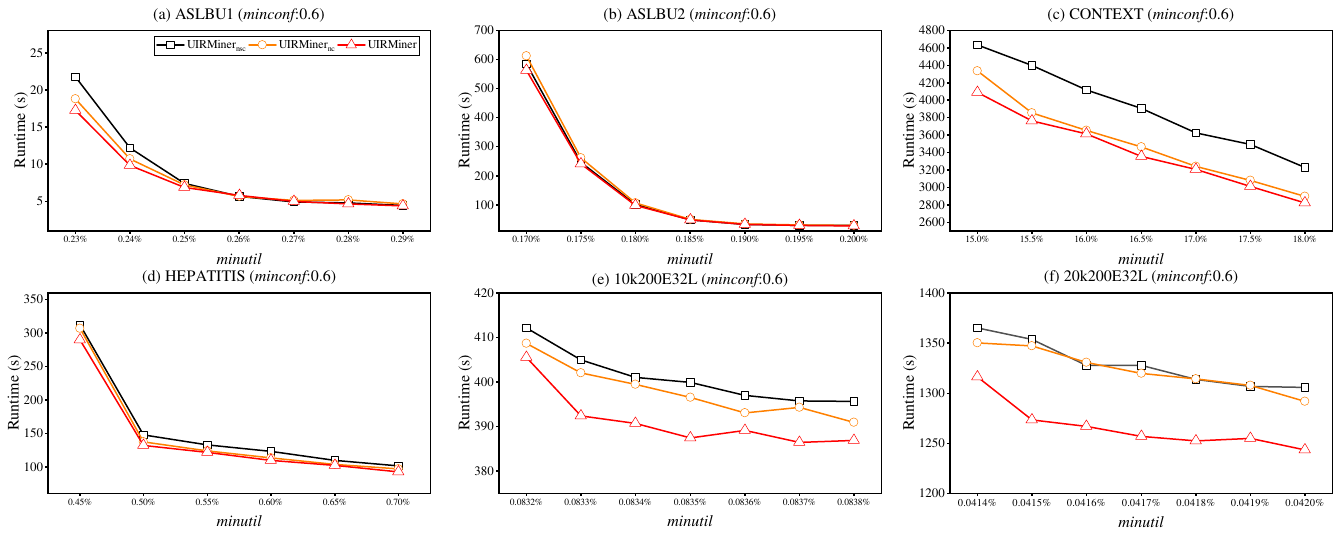}
    \caption{The execution time under various minimum utility thresholds.}
    \label{runtime}
\end{figure*}

\begin{figure*}[htbp]
    \centering  
    \includegraphics[width=1\linewidth]{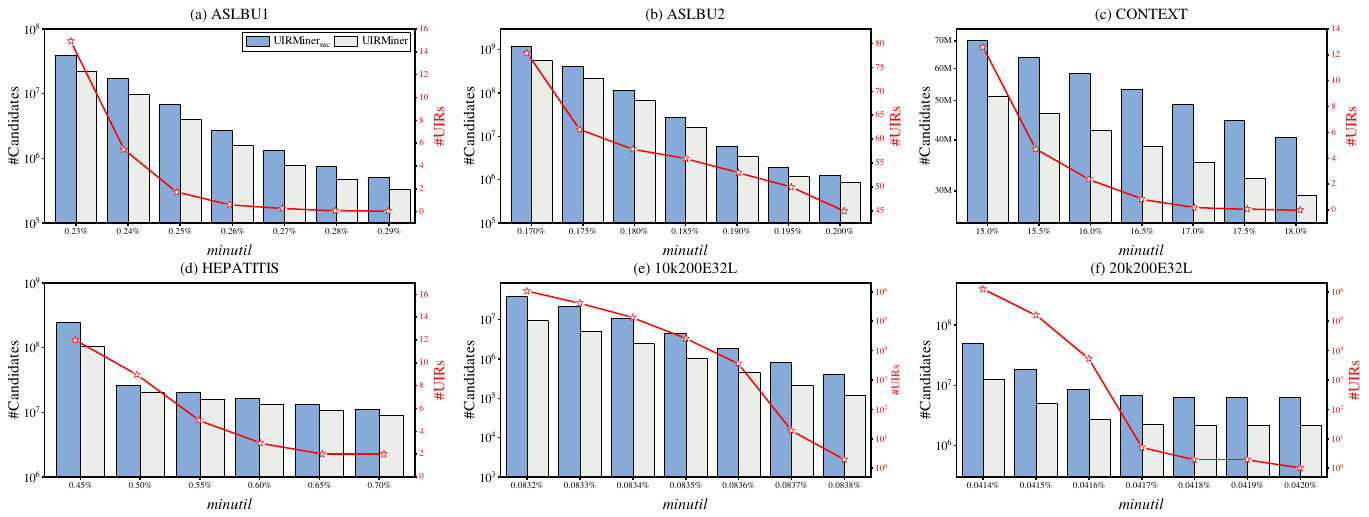}
    \caption{The candidates UIRs obtained by UIRMiner$_{nc}$ and UIRMiner.}
    \label{effectiveness}
\end{figure*}

\subsection{Pruning strategy comparison}

To validate the effectiveness of the proposed pruning strategies, we compare the number of candidate UIRs in this subsection. The number of candidates UIRs can be viewed as the search space of the algorithm, so the smaller the number of candidates, the smaller the search space, and the more efficient the pruning power used by the algorithm. We only use UIRMiner$_{nc}$ and UIRMiner to compare the effectiveness of the pruning strategy since UIRMiner$_{nsc}$ uses the same pruning strategy as UIRMiner. Fig. \ref{effectiveness} reveals the results of UIRMiner$_{nc}$ and UIRMiner.

From Fig. \ref{effectiveness}, we can know that in each dataset, the number of candidate UIRs will decrease as the \textit{minutil} increases. When the \textit{minutil} is small, the number of candidates produced by the two algorithms differs dramatically. When the \textit{minutil} becomes large enough, the number of candidates for the two algorithms is similar in datasets ASLBU1, ASLBU2, and HEPATITIS. But the number of candidates extracted by UIRMiner is always smaller than UIRMiner$_{nc}$. In Fig. \ref{effectiveness} (c), (e), and (f), the difference in the number of candidates generated by the two algorithms is consistently large. Especially in Fig. \ref{effectiveness} (f), no matter how \textit{minutil} changes, UIRMiner$_{nc}$ always produces several times more candidates than UIRMiner. For example, in Fig. \ref{effectiveness} (f), when the \textit{minutil} is 0.0414\%, the number of candidate UIRs of UIRMiner$_{nc}$ and UIRMiner is respective near to $2$ $\times$ $10^8$ and $7$ $\times$ $10^7$; when the \textit{minutil} is 0.042\%, and the number of candidate UIRs of UIRMiner$_{nc}$ and UIRMiner is respective near to $6$ $\times$ $10^6$ and $2$ $\times$ $10^6$. We can also find a similar result in Fig. \ref{runtime}. With a small \textit{minutil}, UIRMiner will take much less time to extract the UIRs; while in a large \textit{minutil}, UIRMiner and UIRMiner$_{nc}$ take nearly the same time to complete the discovery task in datasets ASLBU1, ASLBU2, and HEPATITIS. But in datasets CONTEXT, 10k200E32L, and 20k200E32L, there is always a large difference in execution time between UIRMiner and UIRMiner$_{nc}$.

In conclusion, combining the result of the number of candidates and the execution time of algorithms, we can prove that the pruning strategy, namely utility complement pruning, works well.

\subsection{Memory and storage analysis}

\begin{table*}[htbp]
    \centering
    \caption{Experiment results in CONTEXT} 
    \label{ESC}
    \scalebox{1} {
    \begin{tabular}{|c|c|ccccccc|}
    \hline
    CONTEXT & \textit{minutil} & 15\% & 15.5\% & 16\% & 16.5\% & 17\% & 17.5\% & 18\% \\  \hline 
            \multirow{2}{*}{Runtime (sec)} & UIRMiner$_{nsc}$ & 4629.342 & 4214.348 & 4123.856 & 3772.562 & 3661.318 & 3426.444 & 3242.622 \\
            & UIRMiner & 4144.028 & 3769.216 & 3548.719 & 3278.076 & 3089.511 & 2947.311 & 2698.941 \\
    	\hline 
            \multirow{2}{*}{Memory (MB)} & UIRMiner$_{nsc}$ & 319.84 & 318.94 & 318.22 & 317.71 & 316.89 & 316.33 & 315.75 \\
            & UIRMiner & 315.99 & 315.67 & 315.06 & 314.69 & 314.64 & 313.98 & 313.71 \\
    	\hline 
            \multirow{2}{*}{File size (B)} & UIRMiner$_{nsc}$ & 342206 & 131658 & 68862 & 25025 & 6088 & 3301 & 923 \\
            & UIRMiner & 219676 & 83907 & 43657 & 15936 & 3935 & 2160 & 689 \\
    \hline 
    \end{tabular}
 }
\end{table*}

\begin{figure*}[htbp]
    \centering
    \includegraphics[width=1\linewidth]{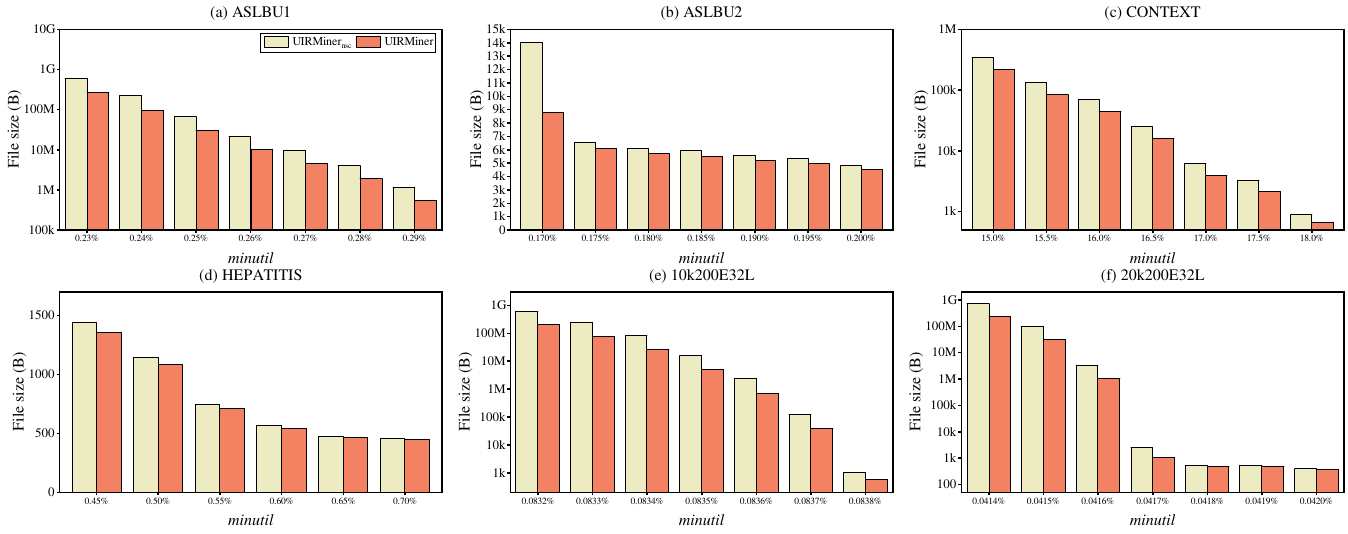}
    \caption{The output file size of the proposed methods under various minimum utility thresholds.}
    \label{memory}
\end{figure*}

\begin{figure*}[htbp]
    \centering
    \includegraphics[width=0.95\linewidth]{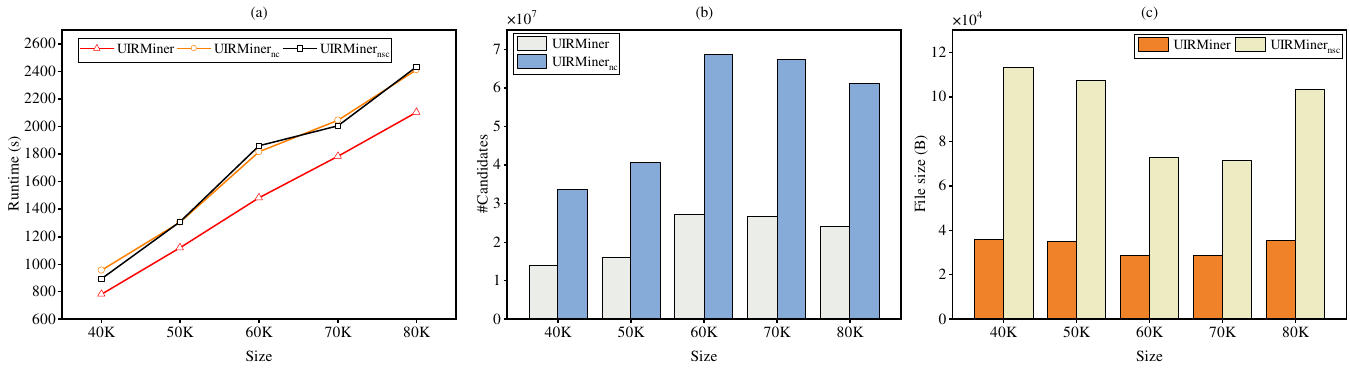}
    \caption{The scalability test between different methods.}
    \label{scalability}
\end{figure*}

In this subsection, we compare the memory usage and the size of the output file of different algorithms to demonstrate the effectiveness and efficiency of UIRMiner. In Table \ref{ESC}, the memory consumption of dataset CONTEXT is presented. We can find that the memory usage results of the different algorithms and various \textit{minutil} are almost the same. Thus, we do not present the other dataset's memory usage results. Fig. \ref{memory} depicts the output file size of the algorithms UIRMiner$_{nsc}$ and UIRMiner. Since UIRMiner$_{nc}$ uses the same relation representation, we do not compare the result of UIRMiner$_{nc}$ with UIRMiner. In ASLBU1, the number of UIRs obtained by the algorithm is large. Therefore, if we use the matrix relation representation, the storage of output files will take up a lot of disk space. However, the file size with a numerical encoding relation representation is half that of the traditional method. ASLBU2 is the newer version of ASLBU1. Thus, the file size of the numerical encoding relation representation is also less than that of the matrix representation. However, it won't be much less because the number of UIRs is small when \textit{minutil} becomes large. Fig. \ref{memory} (c) and (d) also follow this trend since the output UIRs are relatively small. Thus, the file size of the output result is almost the same, but the numerical encoding relation representation is always the smaller one. In the two synthetic datasets, when there is a small \textit{minutil} the output file size of UIRMiner$_{nsc}$ is much larger than UIRMiner. As the \textit{minutil} increases, the difference in output file size becomes smaller. However, the result from UIRMiner is still the lesser one. In summary, a numerical encoding relation representation is always better than the traditional relation representation since the former can save a lot of storage, which is quite friendly for devices with limited storage space, such as embedded devices.

\subsection{Scalability test}

To validate the algorithms' effectiveness on large datasets, we conducted several experiments on the size of each dataset, varying from 40k to 80k. To avoid choosing a suitable \textit{minutil}, we will discover the top 200 UIRs of each dataset according to utility.

Fig. \ref{scalability} shows the results of the compared algorithms, which involve execution time, candidate UIRs, and storage usage. With the size of the dataset boosted, their execution time increases linearly. In Fig. \ref{scalability} (a), we can find that UIRMiner in any size of a dataset outperforms the comparators, i.e., UIRMiner$_{nc}$ and UIRMiner$_{nsc}$. We can discover that although UIRMiner$_{nsc}$ uses the \textit{complement pruning strategy}, its runtime is similar to UIRMiner$_{nc}$. It reveals that the numerical encoding relation representation can reduce the execution time. Comparing the results of UIRMiner and UIRMiner$_{nc}$, we can derive that the complement pruning strategy works well. In Fig. \ref{scalability} (b), the number of candidate UIRs generated from UIRMiner and UIRMiner$_{nc}$ verifies, while UIRMiner is better than UIRMiner$_{nc}$. The number of candidate UIRs can be viewed as the search space of the algorithm. Thus, the fewer candidates there are, the less execution time there will be. We also compared the storage usage between UIRMiner and UIRMiner$_{nsc}$ in Fig. \ref{scalability} (c). From the result, we can validate that the numerical encoding relation representation can effectively reduce the storage space of the file. In summary, the proposed algorithm, UIRMiner, can effectively cope with large-scale datasets.
\section{Conclusion}   \label{sec:conclusion}

In this paper, we define the problem of discovering utility-driven interval rule mining and propose UIRMiner, which is the first algorithm to extract all utility-driven interval rules from a given interval database. To save storage of the output file, we first introduced the numerical encoding relation representation to substitute the traditional matrix relation representation. In addition, we incorporate the complicated relation between intervals and the utility upper bound to create a \textit{utility complement pruning} strategy to reduce the search space. To maintain the candidate UIRs, we use two data structures, namely \textit{AUL} and \textit{CUL}. The abundant experimental results conducted on both real and synthetic datasets show the proposed algorithm is better than the methods that do not use any optimization.

In the future, we will look forward to applying UIRMiner to real-life scenarios, such as service recommendation and server optimization analysis. Target rule mining is also an interesting research field that combines target mining and rule mining. It can derive rules for desired targets, which can then be better used for recommendations. Besides, we want to use UIRMiner to research the area of anomaly detection using interval rules according to rare high-risk rules \cite{gan2021anomaly}.

\ifCLASSOPTIONcompsoc
  \section*{Acknowledgments}
\else
  \section*{Acknowledgment}
\fi

This work was supported in part by the National Natural Science Foundation of China (Nos. 62002136 and 62272196), Natural Science Foundation of Guangdong Province of China (Nos. 2020A1515010970 and 2022A1515011861), and Shenzhen Research Council (Nos. JCYJ20200109113427092 and GJHZ20180928155209705).

\ifCLASSOPTIONcaptionsoff
  \newpage
\fi

\bibliographystyle{IEEEtran}
\bibliography{UIRMiner.bib}

\end{document}